\pdfoutput=1

\documentclass[11pt]{article}

\usepackage[final]{acl}

\usepackage{times}
\usepackage{latexsym}

\usepackage[T1]{fontenc}

\usepackage[utf8]{inputenc}

\usepackage{microtype}

\usepackage{inconsolata}

\usepackage{graphicx}

\usepackage{amsmath}
\usepackage{amssymb}
\usepackage{mathtools}
\usepackage{amsthm}
\usepackage{booktabs}
\usepackage{multirow}

\theoremstyle{plain}
\newtheorem{theorem}{Theorem}[section]

\theoremstyle{definition}

\theoremstyle{remark}

\usepackage{enumitem}
\usepackage{color}

\usepackage{algorithm}
\usepackage{algorithmic}

\title{Supervised Optimism Correction: Be Confident When LLMs Are Sure}

\author{
 \textbf{ Junjie Zhang\textsuperscript{1}\thanks{Equal contribution.}}, 
 \textbf{ Rushuai Yang\textsuperscript{2}\footnotemark[1]}, 
 \textbf{ Shunyu Liu\textsuperscript{1}},
 \textbf{ Ting-En Lin\textsuperscript{3}},
 \textbf{ Fei Huang\textsuperscript{3}},
 \\ 
 \textbf{ Yi Chen\textsuperscript{2}},
 \textbf{ Yongbin Li\textsuperscript{3}\thanks{Corresponding authors.}},
 \textbf{ Dacheng Tao\textsuperscript{1}\footnotemark[2]}
\\
\\
 \textsuperscript{1}Nanyang Technological University, Singapore\\
 \textsuperscript{2}Hong Kong University of Science and Technology\\
 \textsuperscript{3}Tongyi Lab
\\
}

\begin{document}
\maketitle
\begin{abstract}

In this work, we establish a novel theoretical connection between supervised fine-tuning and offline reinforcement learning under the token-level Markov decision process, revealing that large language models indeed learn an implicit $Q$-function for inference.
Through this theoretical lens, we demonstrate that the widely used beam search method suffers from unacceptable over-optimism, where inference errors are inevitably amplified due to inflated $Q$-value estimations of suboptimal steps. 
To address this limitation, we propose \textit{\textbf{S}upervised \textbf{O}ptimism \textbf{C}orrection}~(SOC), which introduces a simple yet effective auxiliary loss for token-level $Q$-value estimations during supervised fine-tuning. 
Specifically, the auxiliary loss employs 
implicit value regularization
to boost model confidence in expert-demonstrated responses, thereby suppressing over-optimism toward insufficiently supervised responses.
Extensive experiments on mathematical reasoning benchmarks, including GSM8K, MATH, and GAOKAO, showcase the superiority of the proposed SOC with beam search across a series of open-source models.

\end{abstract}

\section{Introduction}

Recent advances in Large Language Models~(LLMs) have demonstrated remarkable success across diverse tasks such as instruction following~\cite{brown2020language,zhou2023instruction,taori2023alpaca,tao2024survey,liu2025survey}, code generation~\cite{liu2023rltf,le2022coderl,nijkamp2022codegen,jiang2024survey}, and medical diagnosis~\cite{zhang2023alpacare,wang2023chatcad}.
Within these developments, complex reasoning capabilities have attracted increasing attention from research communities, attributed to their capability of enabling LLMs to tackle intricate problem-solving~\cite{wei2022chain,yao2022react,kojima2022large}.
Despite these achievements, learning to reason remains a critical yet challenging task for LLMs, particularly for smaller models with limited parameters~\cite{liu2025can,zhang2025reasoning,zhang2025r1}. The inherent complexity stems from inefficient exploration of LLMs, as the combinatorial nature of the vocabulary space results in an exponential growth of potential reasoning paths~\cite{snell2024scaling}.

A prevalent paradigm for enhancing reasoning capabilities involves Supervised Fine-Tuning~(SFT) on high-quality demonstration data~\cite{brown2020language,yao2024mulberry,yang2024qwen2}, where models learn to imitate expert reasoning patterns through next-token prediction. The complementary strategies employ search-based decoding techniques during inference~\cite{snell2024scaling,xie2024self,ding2025dynamic}, such as beam search, which aims to enhance reasoning by exploring multiple candidate pathways. However, there is often an overlooked disconnect between the local token-level optimization of SFT and the global sequence-level objectives pursued by search-based decoding. 
This disconnect results in a critical misalignment: while SFT focuses on maximizing the likelihood of individual token predictions, search-based decoding operates by scoring entire sequences, which may not directly align with the local training objectives of LLMs.
By addressing this gap, we can potentially unlock more robust reasoning capabilities in LLMs, aligning training objectives more closely with inference-time goals.

\begin{figure*}[th]
  \includegraphics[width=\linewidth]{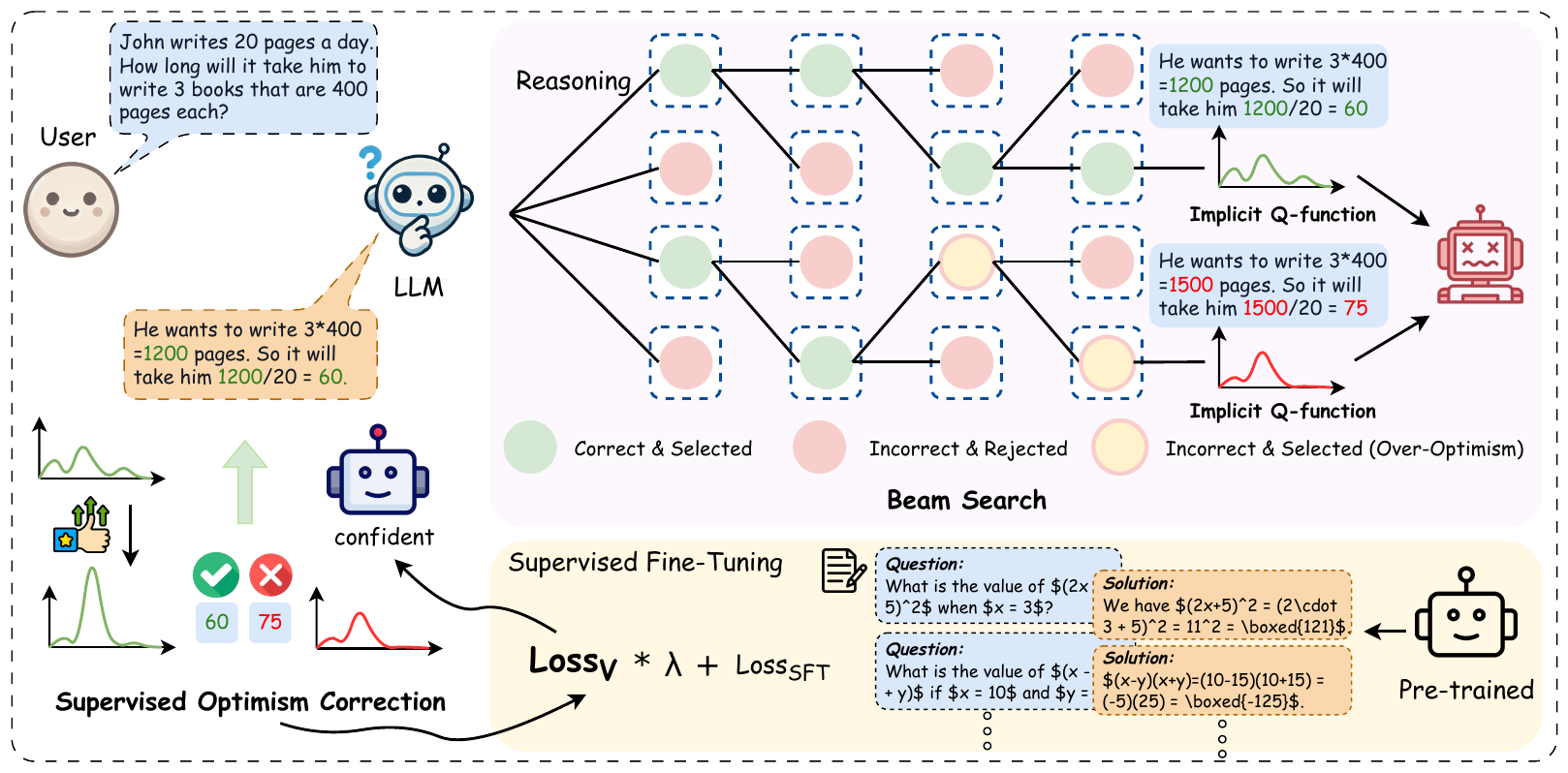}
  \caption {Illustration of Supervised Optimism Correction. Guided by an implicit $Q$-function, beam search suffers from the over-optimism problem during the decoding process, which confuses LLM for reliable response choice. In particular, the over-optimism can amplify errors during beam search, leading to the selection of incorrect trajectories with higher $Q$-values. To alleviate this problem, SOC introduces an auxiliary loss during SFT to boost model confidence in expert-demonstrated responses, about which LLM should be sure.}
    \label{fig: illustration example}
\end{figure*}

In this work, we employ the token-level Markov decision process to establish a novel theoretical connection between SFT and offline Reinforcement Learning~(RL)~\cite{levine2020offline,kumar2020conservative,kostrikov2021offline,lyu2022mildly}. We theoretically demonstrate that LLMs indeed learn an implicit $Q$-function to estimate the expected utility of token sequences during SFT. Through this lens, we further explore the over-optimism problem in the widely used beam search method, revealing that the search process disproportionately favors sequences with inflated $Q$-value estimations. This over-optimism arises because the beam search method autoregressively selects tokens with locally overestimated Q-values, thereby inevitably amplifying errors through cascading suboptimal steps.

To alleviate this problem, we propose Supervised Optimism Correction~(SOC), which introduces a simple auxiliary loss during SFT to give supervised responses a state-value bonus. Unlike prior RL-based methods that require explicit reward modeling, SOC operates purely within the SFT paradigm by imposing implicit value regularization. This regularization boosts confidence in expert-demonstrated data while potentially penalizing high $Q$-value estimations for insufficiently supervised reasoning steps. As a result, the model learns to autonomously prune low-quality reasoning paths during inference without relying on external verifiers or reward models. For instance, when encountering erroneous intermediate steps, the corrected $Q$-values suppress further exploration of those branches, mirroring human-like error recognition and recovery patterns, requiring no architectural modifications or additional inference-time computations.
Our core contributions are summarized as follows:
\begin{itemize}[leftmargin=*]
\item We establish a novel theoretical connection between SFT and offline RL, identifying and formalizing the over-optimism problem of implicit $Q$-functions for LLM beam search.
\item We develop SOC, a lightweight yet effective method that boosts model confidence in expert-demonstrated data through an auxiliary value regularization loss during SFT.
\item Extensive experiments demonstrate the effectiveness of SOC in mathematical reasoning benchmarks GSM8K, MATH, and GAOKAO, significantly improving the performance of open-source models like Qwen-2-1.5B, Qwen-2.5-3B, and Qwen-2.5-7B.
\end{itemize}

\label{intro}

\section{Preliminaries}
\subsection{Token-level MDP for LLMs}

We start by formulating the token generation process of LLMs as a token-level Markov Decision Process (MDP) \citep{DPO, from_r_to_q, zhong2024dpo, zhong2022gec}, enabling a structured analysis of their decision-making dynamics and reasoning capabilities. An MDP \cite{intro_RL} is typically defined by a tuple $(\mathcal{S}, \mathcal{A}, \mathcal{P}, R, \gamma)$, where $\mathcal{S}$ denotes the state space, $\mathcal{A}$ the action space, $\mathcal{P}$ the transition probability function, $R$ the reward function, and $\gamma \in [0,1]$ the discount factor, adopting a discount factor of $\gamma = 1$ throughout the paper. In our formulation, the state $s \in \mathcal{S}$ corresponds to the token generation context at step $t$, represented as 
$s_t = (x_0, x_1, \dots, x_t),$ which serves as the conditioning context for generating the next token $x_{t+1}$. The action space $\mathcal{A}$ is defined by the fixed vocabulary from which the next token is selected. The probability distribution over actions is parameterized by the LLM’s learned policy $\pi_\theta(a_t | s_t)$, with $a_t \in \mathcal{A}$ at each step $t$. In addition, the transition function $P(s_{t+1} | s_t, a_t)$ models the concatenation of the current state with the chosen action to form the subsequent state, \textit{i.e.} $s_{t+1} = (s_t, a_t)$. In the offline RL setting \citep{levine2020offline, kumar2020conservative, kostrikov2021offline}, the objective is to obtain a parameterized policy $\pi_\theta$ that maximizes the expected cumulative reward using an offline dataset $\mathcal{D}$. This can be formulated as:
\begin{equation}
\label{eq:offline_RL_objective}
    \max_{\theta} \mathbb{E}_{s_t\sim\mathcal{D}, a_t\sim \pi_\theta} \left[\sum_{t=0}^{T} \gamma^t R(s_t, a_t)\right],
\end{equation}
where the reward function $R(s_t, a_t)$ measures the quality of the generated output. In tasks such as mathematical reasoning, the reward is often sparse, with a terminal reward $R_{\text{outcome}} = 1$ indicating the correctness of the final result, while intermediate rewards are set to zero. When adopting a discount factor of $\gamma = 1$, the gradient of Equation~\eqref{eq:offline_RL_objective} with respect to the policy parameters is given by:
\begin{equation}
\begin{aligned}
\label{eq: gradient of maximum return}
&\mathbb{E}_{s_t\sim\mathcal{D}, a_t\sim \pi_\theta} \left[ \sum_{t=0}^{T} \nabla_{\theta} \log \pi_\theta(a_t | s_t) R(\tau) \right].
\end{aligned}
\end{equation}
According to the sparse reward assumption in the above MDP setting for correct $\tau$:
\begin{equation}
    R(\tau) := \sum_{t=0}^{T} \gamma^t R(s_t,a_t) = R_\text{outcome} = 1,
\end{equation} then Equation~\eqref{eq: gradient of maximum return} can be simplified to: 
\begin{equation}
    \begin{aligned}
    \label{eq: gradient of maximum return1}
    &\mathbb{E}_{s_t\sim\mathcal{D}, a_t\sim \pi_\theta} \left[\sum_{t=0}^{T} \nabla_{\theta} \log \pi_\theta(a_t | s_t)\right].
    \end{aligned}
\end{equation}
On the other hand, the pretrained model $\pi_\theta$ in the SFT stage aims to imitate the behavior of the expert policy $\pi^\star$. This can be expressed as
\begin{equation}
    \pi_{\theta} = \operatorname*{argmax}_{\theta} \mathbb{E}_{(s^\star,a^\star) \sim \mathcal{D}} \left[ \log \pi_\theta(a^\star | s^\star) \right].
\end{equation}
Note that the optimization problem above is a special case of the optimization problem encountered in offline reinforcement learning. The reward signal comes from matching the expert’s demonstrations. In addition, to avoid repetitive or suboptimal outputs, we can maximize the entropy of the policy $\mathcal{H}(\pi_\theta)$ simultaneously to encourage exploration and prevent the policy from becoming overly deterministic, the optimization problem will be written as
\begin{multline}
\label{eq: fomulate as soft RL}
    \pi_{\theta} = \operatorname*{argmax}_{\theta} \mathbb{E}_{(s^\star,a^\star) \sim \mathcal{D}} \left[ \log \pi_\theta(a^\star | s^\star) \right] \\
    + \mathcal{H}\left(\pi_\theta\left(\cdot|s^\star\right)\right),
\end{multline}
which is equivalent to a maximum entropy RL objective \citep{ziebart2010modeling,softsac} but with the reward coming from the expert's actions rather than an environment-based reward signal.

\subsection{Deriving the $Q$-function as LLM Logits}
In the general maximum entropy RL setting, the fixed-point solution of Equation~\eqref{eq: fomulate as soft RL} is given by \citep{ziebart2010modeling, from_r_to_q, guo2021efficient} as:
\begin{equation}
\label{eq: SQL equation from citation}
   \pi^\star(a | s) = \exp\left(Q^\star(s, a) - V^\star(s)\right),
\end{equation}
where $Q^\star(s, a)$ is the optimal $Q$-function, representing the accumulated reward starting from state $s$, taking action $a$, and following the optimal policy thereafter. The optimal value function $V^\star(s)$ is related to $Q^\star(s, a)$ and is given by:
\begin{equation}
\label{eq: define V}
   V^\star(s) = \log \sum_a  \exp(Q^\star(s, a)).
\end{equation}
By combining Equation~\eqref{eq: SQL equation from citation} and Equation~\eqref{eq: define V} and taking the logarithm of both sides, we obtain:
\begin{equation}
\label{eq: SQL equation}
   \pi^\star(a | s) = \frac{\exp(Q^\star(s, a))}{\sum_{a'} \exp(Q^\star(s, a'))}.
\end{equation}
Thus, the optimal policy is derived from the softmax of the corresponding $Q$-function. On the other hand, in the context of pretrained models in LLMs, the policy is also obtained from the softmax of logits:
\begin{equation}
\label{eq: softmax logits}
   \pi_\theta(a | s) = \frac{\exp(Q_\theta(s, a))}{\sum_{a'}  \exp(Q_\theta(s, a')) },
\end{equation}
where $Q_\theta(s, a)$ represents the logits generated by the pretrained model. These logits implicitly approximate the optimal $Q$-function, ensuring that the parameterized policy $\pi_\theta$ closely resembles the optimal policy $\pi^\star$. We will refer to LLM logits as the implicit $Q$-function throughout the paper.

\subsection{Beam Search Decoding}

Beam search is a heuristic search algorithm widely used in decoding for LLMs' test time \citep{pascual2020directed, sun-etal-2023-allies, liu2025can}. Given a pre-trained language model $\pi_\theta$ that generates tokens in an autoregressive manner, standard beam search aims to approximate the most probable output sequence by maintaining a fixed-size set of candidate sequences (beams) at each decoding step. Formally, at each step $t$, the method expands all hypotheses by considering the top-$k$ best candidate tokens according to their accumulated log-probability scores. This process continues until an end-of-sequence (EOS) token is generated or the maximum sequence length $T$ is reached. The algorithm is presented in Algorithm~\ref{alg:beam_search}.
\begin{algorithm}[h]
    \caption{Standard Beam Search Decoding}
    \label{alg:beam_search}
    \begin{algorithmic}[1]
        \REQUIRE Language model $\pi_\theta(a_t | s_t)$, question input $x_0$, beam width $k$, maximum sequence length $T$
        \STATE Initialize beam set $\mathcal{B} \gets \{(s_0, v_0)\}$, where $s_0 = x_0$ and score $v_0=0$
        \FOR{$t = 1$ to $T$}
            \FOR{each $(s_t, v_t) \in \mathcal{B}$}
                \STATE Compute next-token probabilities $\pi_\theta(a_t \mid s_t)$ for $a_t \in \mathcal{A}$
                \STATE Select the top-$k$ tokens $\mathcal{A}_t^\text{top}$ from $\mathcal{A}$ based on the accumulated score: 
                \STATE \quad \quad $v_{t+1} = v_t + \log \pi_\theta(a_t | s_t)$
                \FOR{each $a_t \in \mathcal{A}_t^\text{top}$}
                    \STATE Compute new state $s_{t+1} = (s_t, a_t)$
                    \STATE Add corresponding $(s_{t+1}, v_{t+1})$ to $\mathcal{B}$
                \ENDFOR
            \ENDFOR
        \ENDFOR
        \STATE Return best sequence: $\operatorname*{argmax}_{(s_T, v_T) \in \mathcal{B}} v_T$
    \end{algorithmic}
\end{algorithm}
\section{Supervised Optimism Correction}
In Subsection~\ref{sec: Over-optimism problem in Beam Search}, we first discuss the over-optimism problem observed in beam search during inference. Next, Subsection~\ref{sec: Impact of Value Function Estimation Error on Beam Search Performance} investigates a potential cause of this issue by analyzing how inflated $Q$-value estimation errors, particularly those arising from insufficiently supervised states, can amplify over-optimism through the maximization operation of beam search. Finally, Subsection~\ref{sec: V loss for optimism correction} introduces our proposed Supervised Optimism Correction (SOC) method, which incorporates an auxiliary $V$ loss to better align the response selection process with expert-demonstrated responses.
\subsection{Over-optimism Problem in Beam Search}
\label{sec: Over-optimism problem in Beam Search}

In this section, we explore the over-optimism problem in beam search, examining the factors that contribute to its exacerbation and its impact on response quality during inference. Beam search is widely used to generate sequences based on accumulated log-probability scores. However, over-optimism can arise when the search disproportionately favors sequences with inflated $Q$-value estimates, leading to suboptimal results. This phenomenon is also prevalent in traditional reinforcement learning contexts \citep{Thrun1999IssuesIU, van2016deep, kumar2020conservative, kostrikov2021offline, wen2024towards}. To formalize this issue, we express the beam search selection process in terms of the $Q$-function. At every intermediate step $T$, by incorporating Equation~\eqref{eq: SQL equation}, the accumulated log-probability over the sequence can be expressed as:
\begin{align}
\label{eq: beam search, logsum to Q}
&\sum_{t=0}^{T} \log \pi(a_t | s_t) =
\sum_{t=0}^{T} \left(Q(s_t,a_t)-V(s_t) \right)
\\ =& \sum_{t=0}^{T - 1} \left(Q(s_t,a_t)-V(s_{t+1})\right) + Q(s_T,a_T) \\& ~ \quad \quad - V(s_0). \nonumber
\end{align}
This can be further simplified using the Bellman equation for the value function for all $t < T$:
\begin{equation}
Q(s_t,a_t) = \mathbb{E}_{s_{t+1}}[R(s_t, a_t) + \gamma V(s_{t+1})].
\end{equation}
Under the assumptions that the transition dynamics $P(s_{t+1} |  s_t, a_t)$ are deterministic, $\gamma = 1$, and intermediate rewards are zero by default, the accumulated log-probability simplifies to:
\begin{align}
\sum_{t=0}^{T} \log \pi(a_t | s_t) = Q(s_T,a_T) - V(s_0).
\end{align}
Since all candidates share the same question input $s_0$, $V(s_0)$ is same and beam search relies heavily on these $Q$-values of the generated token to select the top-$k$ candidates at each step, as described in Equation~\ref{eq: beam search, logsum to Q}. However, the process may include candidates with inaccurate $Q$-value estimates $Q(s_T, a_T)$, which can result from limited supervision during the fine-tuning stage. When $Q$-values are overestimated in insufficiently supervised states, this overestimation can amplify errors during beam search, leading to the selection of suboptimal trajectories with higher $Q$-values, even if they are not aligned with the most reliable, well-supervised paths. Consequently, the final output is biased toward these less reliable paths. Figure~\ref{fig: illustration example} illustrates a case example of the over-optimism problem. In the next subsection, we investigate the causes of over-optimism in beam search, with a particular focus on how $Q$-value estimation errors contribute to its exacerbation and the resulting impact on beam search performance.

\subsection{Impact of Value Function Estimation Error on Beam Search Performance}
\label{sec: Impact of Value Function Estimation Error on Beam Search Performance}
To analyze how the estimation error of $Q$ affect the sampling process in inference time, we start to obtain the key observation from the gradient of the cross-entropy loss. Recall that in SFT stage, the model is trained to align its predicted distribution $\pi_\theta$ with the target distribution $\pi^\star$ by minimizing the cross-entropy loss \citep{le2022coderl, qwen2.5math}:
\begin{equation}
  \mathcal{L}_{\text{SFT}} = \mathbb{E}_{s\sim \mathcal{D}}\left[-\sum_{a} \pi^\star(a|s) \log \pi_\theta(a|s) \right],
\end{equation}where target distribution $\pi^\star$ could be one-hot encoding as dataset $\mathcal{D}$ indicates the set of high-quality demonstrations. By incorporating Equation~\eqref{eq: SQL equation} into the gradient form of $\mathcal{L}_{\text{SFT}}$, we can see that the $\mathcal{L}_{\text{SFT}}$ minimize the estimation error of implicit $Q$, \textit{i.e.}  

\begin{equation}
    \begin{aligned}
        \label{eq: gradient of sft}
        &\nabla_\theta \mathcal{L}_{\text{SFT}} = \mathbb{E}_{s\sim \mathcal{D}}\left[ -\sum_{a} \frac{\pi^\star(a|s) }{\pi_\theta(a|s)} \frac{\partial \pi_\theta(a|s)}{\partial \theta} \right] 
    \\ = & \mathbb{E}_{s\sim \mathcal{D}}\left[ \sum_{a}  \underbrace{\left(\widetilde{Q}_\theta(s,a) - \widetilde{Q}^\star(s,a) \right)}_{\text{estimation error}}  \frac{\partial Q_\theta(s,a)}{\partial \theta} \right],        
    \end{aligned}
\end{equation}
where $\widetilde{Q}(s,a) \in [0,1]$ refers to normalized value function for each action $a$. We give the detailed proof in Appendix~\ref{app: The Derivation of Estimation Error}. Intuitively, the learned policy $\pi_\theta$ is derived from the softmax of $Q_\theta(s,a)$, so estimation errors in $Q_\theta(s,a)$ will lead to deviations from the optimal policy $\pi^\star$, and the policy may favor suboptimal actions, leading to insufficiently supervised action and misguide to the suboptimal state. This error could inevitably be amplified during inference since the beam search takes action based on the estimated $Q_\theta(s,a)$ as shown in Equation~\eqref{eq: beam search, logsum to Q}.

\subsection{$V$ Loss for Optimism Correction}
\label{sec: V loss for optimism correction}
Based on the analysis in the previous subsection, one approach to mitigating the over-optimism problem is to reduce the impact of estimation error. To this end, we propose an auxiliary objective $V$ loss, defined as:
\begin{equation}
\label{eq: v loss}
  \mathcal{L}_{V} = \mathbb{E}_{s\sim \mathcal{D}}\left[-\log \sum_a \exp Q_\theta(s,a)\right].
\end{equation}
Intuitively, this auxiliary loss serves to boost the overall implicit value function for high-quality, labeled data. By elevating the value estimates associated with supervised trajectories, the model is more likely to preferentially select these trajectories during inference, thereby counteracting the bias introduced by overestimated $Q$ values in insufficiently supervised states. Overall, our total objective in the training stage can written as:
\begin{equation}
  \mathcal{L}_{\text{overall}} = \mathcal{L}_{\text{SFT}} + \lambda \cdot \mathcal{L}_{V}
\end{equation}
with a tuning hyperparameter $\lambda$.
\subsection{Effect of the $V$ Loss Update} 
\label{sec: Effect of the V Loss Update}
To gain a mechanistic understanding of the $V$ loss, we analyze the gradient of its objective function:  
\begin{equation}
\nabla_\theta \mathcal{L}_{\text{V}} = \mathbb{E}_{s\sim \mathcal{D}}\left[-\sum_a \widetilde{Q}_\theta(s,a) \frac{\partial Q_\theta(s,a)}{\partial \theta} \right].
\end{equation}  
This expression reveals how the update is guided by the weighted sum of the $Q$-function gradients for labeled data, where $\widetilde{Q}_\theta(s,a)$ acts as an implicit weighting factor influencing parameter adjustments. We provide an illustrative example for the update effect between $\mathcal{L}_{\text{SFT}}$ and $\mathcal{L}_{\text{V}}$ in Figure~\ref{fig: illustation Q}. Specifically, since 
$\sum_a\widetilde{Q}_\theta(s,a) = 1$ for all state $s$, A larger \(\widetilde{Q}_\theta(s,a)\) for a specific action increases the magnitude of its corresponding gradient contribution, thereby resulting in a more substantial update for \(Q_\theta(s,a)\). Consequently, actions with higher \(\widetilde{Q}_\theta(s,a)\) experience a more significant increase in their $Q$ values during optimization. Compared to the SFT update for a supervised action, the update magnitude induced by the $V$ loss is different. In particular, when $Q(s,a)$ is close to $Q^\star(s,a)$, the gradient from the SFT objective tends to be small, as shown in Equation~\eqref{eq: gradient of sft}, potentially resulting in only minimal updates. In contrast, the $V$ loss continues to increase $Q(s,a)$ for the supervised action, thereby compensating for the estimation gap along the supervised trajectories with appropriated $\lambda$. 
\section{Theoretical Analysis}
In this section, we provide a theoretical analysis of our objective.
\begin{theorem}[Contraction of Value Differences]
\label{thm:convergence_values}
Let $V_\theta(s)$ be the approximate value function of state $s$. Suppose that the value function is always positive for any state $s$. If the objective in SFT includes a additional term in Equation~\eqref{eq: v loss}, then after one step of gradient descent with learning rate $\alpha \in [0,1]$, the gap between adjacent states' values contracts, \textit{i.e.},  
\[
| V'_\theta(s_t) - V'_\theta(s_{t+1}) | \leq | V_\theta(s_t) - V_\theta(s_{t+1}) |, \quad \forall t,
\]
where $V'$ denotes the updated value function after one optimization step.
\end{theorem}
We provide a detailed proof in Appendix~\ref{app: proof}. Theorem~\ref{thm:convergence_values} states that the auxiliary loss encourages the implicit value function of neighboring states to become closer. To further illustrate the benefit of this regularization, consider the policy evaluation problem in a token-level MDP with sparse rewards, where the agent only receives a reward at the end of the trajectory while all intermediate rewards are zero. In this setting, the optimal value function satisfies the Bellman equation:  
\begin{equation}
\begin{aligned}
V^\star(s_t) = r\left(s_t, \pi^\star(a_t | s_t)\right) + V^\star(s_{t+1}).
\end{aligned}
\end{equation}
Since the reward is zero for all intermediate states, it follows that:
\begin{equation}
\begin{aligned}
V^\star(s_t) = V^\star(s_{t+1}), \forall t \text{ (intermediate states)}.
\end{aligned}    
\end{equation}
Thus, in sparse reward settings, the optimal gap between the values of adjacent states is naturally small. By minimizing this gap during training, the auxiliary loss guides the learned value function to better reflect the underlying structure of the sparse reward MDP, rather than artificially transforming the sparse reward into a dense one, thereby facilitating more stable and efficient policy evaluation for sparse reward. Moreover, in sparse reward settings, value information needs to propagate over many steps. By encouraging smoother value estimates between adjacent states, the auxiliary loss effectively reduces variance in value updates, which can lead to more robust value estimation \citep{schulman2015high, schulman2017proximal}.

\begin{figure}[h]
  \includegraphics[width=\linewidth]{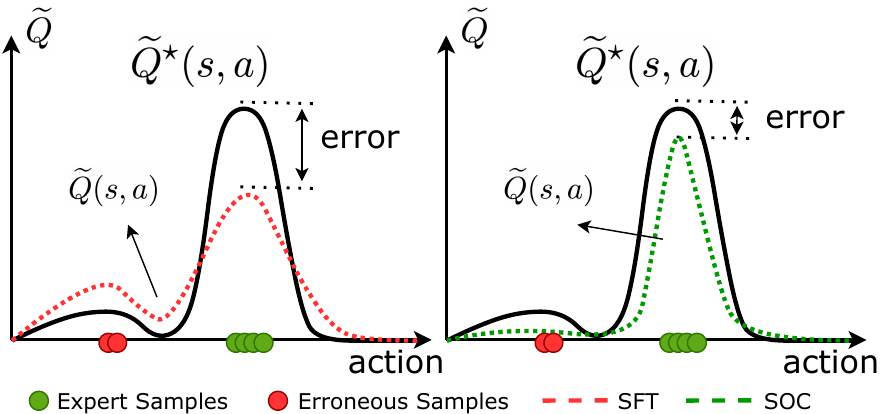}
  \caption {An illustrative example of gradient update magnitude on $Q$-value estimation in SFT and SOC. (Left) In SFT, small gradient updates lead to insufficient supervision, causing persistent overestimation errors and favoring suboptimal actions (red spots). (Right) SOC mitigates overestimation by aligning estimated values with expert demonstrations (green spots), reducing errors, and improving decision-making in inference times.}
    \label{fig: illustation Q}
\end{figure}

\begin{table*}[!t]
\caption{Results (accuracy \% of pass@1) on mathematical reasoning tasks. We report the results of each model sampling with beam search~(beam width of 5) and greedy decoding. For results of GSM8K and MATH-500, models are prompted with 4-shot demonstrations; for GAOKAO-EN, it is zero-shot. There is no sampling process in both beam search and greedy decoding, and therefore, the results are deterministic.}
\label{main-result}
\vskip 0.15in
\begin{center}
\begin{sc}
\resizebox{\textwidth}{!}{%
\begin{tabular}{l|lc|lc|lc}
\toprule
\multirow{2}{*}{Model} & 
\multicolumn{2}{c|}{GSM8K} & 
\multicolumn{2}{c|}{MATH-500} & 
\multicolumn{2}{c}{GaoKao-EN}  \\ 
\cmidrule(lr){2-3} \cmidrule(lr){4-5} \cmidrule(lr){6-7} 
& Beam & Greedy & Beam & Greedy & Beam & Greedy \\ 
\midrule
\textit{Qwen-2-1.5B-Base}         & 56.1   & 48.7  & 25.6  & 23.6    &19.5 &15.8       \\
\quad+SFT(\texttt{DATA-$\beta$})           & 54.9       & 56.4 & 24.0  & 20.2 &17.7 &16.1 \\
\quad+SFT(\texttt{DATA-$\beta$})+SOC        & 56.0\color{red}{(+1.1)}   & 55.6  &25.4\color{red}{(+1.4)} & 21.6 &18.2\color{red}{(+0.5)} &17.1\\
\quad+SFT(\texttt{DATA-$\alpha$})           & 69.1     & 64.6  & 25.0   & 24.0 & 23.9 &22.3 \\
\quad+SFT(\texttt{DATA-$\alpha$})+SOC      & 69.4\color{red}{(+0.3)}   & 64.3 & 26.6\color{red}{(+1.6)}   & 24.4 &27.3\color{red}{(+3.4)} &23.4\\

\midrule
\textit{Qwen-2.5-3B-Base}           & 51.5 &  77.0  & 40.8 & 45.0  &31.4 &41.8    \\
\quad+SFT(\texttt{DATA-$\beta$})                & 52.5     & 76.3  &40.0 & 43.4  &30.1 &41.3\\
\quad+SFT(\texttt{DATA-$\beta$})+SOC          & 75.4\color{red}{(+22.9)}   & 69.4 & 41.6\color{red}{(+1.6)} &36.0 & 30.4\color{red}{(+0.3)} &31.7\\
\quad+SFT(\texttt{DATA-$\alpha$})               & 79.2     & 81.7  &49.6 &49.6  &34.8 &46.5\\
\quad+SFT(\texttt{DATA-$\alpha$})+SOC         & 83.5\color{red}{(+4.3)}   &  80.8 &54.4\color{red}{(+4.8)} &48.4 &46.5\color{red}{(+11.7)} &43.9  \\

\midrule
\textit{Qwen-2.5-7B-Base}               & 58.7  & 85.3  &8.8 &54.6   &33.0 &47.3   \\
\quad+SFT(\texttt{DATA-$\beta$})                & 79.9  & 84.0  &41.0 &47.2 &44.9 &44.2\\
\quad+SFT(\texttt{DATA-$\beta$})+SOC          & 83.7\color{red}{(+3.8)}   & 83.5 &41.4\color{red}{(+0.4)} &48.4  &45.7\color{red}{(+0.8)} &44.2\\
\quad+SFT(\texttt{DATA-$\alpha$})               & 54.7    & 87.3 &41.4 &56.0 &28.6 &51.9\\
\quad+SFT(\texttt{DATA-$\alpha$})+SOC         & 87.9\color{red}{(+33.2)}     & 86.7  &60.4\color{red}{(+19.0)}  &53.4  &51.4\color{red}{(+22.8)} &49.1 \\

\bottomrule
\end{tabular}
}
\end{sc}
\end{center}
\vskip -0.1in
\end{table*}

\section{Experiments}
\label{exps}

In this section, we empirically demonstrate the effectiveness of SOC on the mathematical reasoning ability of LLMs. In Subsection~\ref{subsec:datasets} and~\ref{subsec:setup}, we introduce the datasets and setup of our experiments. We present the results of LLMs scaling from 1.5B to 7B parameters in Subsection~\ref{subsec:mainresults} and provide an ablation study in Subsection~\ref{subsec:ablation}.

\subsection{Datasets}
\label{subsec:datasets}
We conduct experiments on mathematics problem datasets: (1)~GSM8K~\cite{gsm_data}, which consists of 8.8K high quality linguistically diverse grade school math word problems split into 7.5K training set and 1.3K test set, and (2)~MATH~\cite{math_data}, which consists of problems from mathematics competitions with 5 difficulty level with 7.5K training set and 5.0K test set. (3)~We additionally include a recently released dataset STEP-DPO-10K~\cite{stepdpo} as the training dataset.

In the experiments, we adopt a filtered STEP-DPO-10K and a filtered MATH dataset as the training set. Although the STEP-DPO-10K is curated to demonstrate a step-level preference of solutions, we only chose the full and preferred samples and regard them as supervised demonstrations as they are correct and complete question and answer pairs, consisting of 10.3K samples. Correspondingly, we chose the samples with the highest difficulty level (level 5) as the training dataset of MATH, which consists of 2.8K samples. Hereafter, we refer to the filtered STEP-DPO-10K and the filtered MATH as \texttt{DATA-$\alpha$} and \texttt{DATA-$\beta$}. As many previous works, we use the GSM8K, MATH-500~\cite{lightman2023lets}, and GaoKao2023En~\cite{liao-etal-2024-mario} test sets for evaluation.

\subsection{Experimental Setup}
\label{subsec:setup}
In our experiments, we train a series of Qwen base models, Qwen-2-1.5B-Base, Qwen-2.5-3B-Base, and Qwen-2.5-7B-Base. We conduct SFT with high-quality math demonstrations to activate the base models' math reasoning ability to handle math problems. The \texttt{DATA-$\alpha$} and \texttt{DATA-$\beta$} datasets are used independently as two distinct training settings to show the robustness of our method. 

We utilize the accuracy on the test set of GSM8K, MATH-500, and GaoKao2023En as the evaluation metric. Specifically, LLMs are prompted with a few shots CoT of math problem solutions and output format requirements, such as generating the final answer in the \texttt{boxed\{\}}. LLMs predict the solutions with beam search and we compare the predicted answer with the ground truth answer to calculate the accuracy. Following the previous work~\cite{qwen2.5math}, we use the same few shots CoT prompt and set the maximum generation length as 2048. Also, we apply the chat template for structural format during the training, which consists of some special tokens like \texttt{<|im\_start|>} and \texttt{<|im\_end|>}. For evaluation, we use two kinds of CoT prompts~(refer to Appendix~\ref{app: prompt} for details), one consists of the special tokens of SFT, and another does not. We report the better results of SFT models among the two prompts and use the same prompt for SOC, which ensures a fair comparison and reliable results regarding the impact of prompts. 

\subsection{Main Results}
\label{subsec:mainresults}
The goal of our empirical results is to show the effectiveness of SOC by improving the performance of beam search. In the main experiments, we SFT the base models by \texttt{DATA-$\alpha$} and \texttt{DATA-$\beta$} with auxiliary loss of SOC. The results are shown in Table \ref{main-result}. SOC significantly improves the beam search performance over SFT on both training datasets, evaluated on three benchmarks. Despite its simplicity, SOC consistently improves the performance of SFT models in beam search decoding and gains even more than 20 percent accuracy improvement in several settings by generating more confident and reliable reasoning steps during the beam search. 

From the results of the base models, we find that beam search has far worse performance than greedy decoding~(refer to the example in Subsection~\ref{subsec:ablation}). This issue stems from inaccurate $Q$-value estimation. It exerts an insignificant impact on greedy sampling if it does not alter the action of each individual step. However, it significantly degrades beam search results because it influences evaluation of the entire sequence and gets amplified during the inference process. By mitigating the impact of value function estimation error with optimism correction, SOC improves the performance of beam search and significantly alleviates this issue.

\subsection{Ablation}
\label{subsec:ablation}
\paragraph{Beam Search Width} To investigate the impact of the beam search width on performance and how SOC corrects the over-optimism in various searching spaces, we compare the Base, SFT, and SOC across different beam widths. The results are shown in Figure~\ref{fig: ablation}. We observe that:~(1)~on GSM8K, SFT fails to improve the performance of the Base model on beam search, instead, slightly degrading it. In contrast, SOC significantly enhances performance, highlighting its superiority.~(2)~on MATH-500, Base model has a limited performance compared to its accuracy of greedy decoding~(54.6\%)~as shown in Table~\ref{main-result}, demonstrating the harmful effect of the over-optimism problem on beam search.~(3)~on both benchmarks, SOC consistently outperforms Base model and SFT across beam widths ranging from small to large, validating that after optimism correction, the correct reasoning candidate consequences are favored in the beam search process even with small searching space. The results exhibit that, with supervised optimism correction, beam search can find better responses in fewer search branches, thereby reducing the search space and inference cost.
\begin{figure}[t]
  \includegraphics[width=\linewidth]{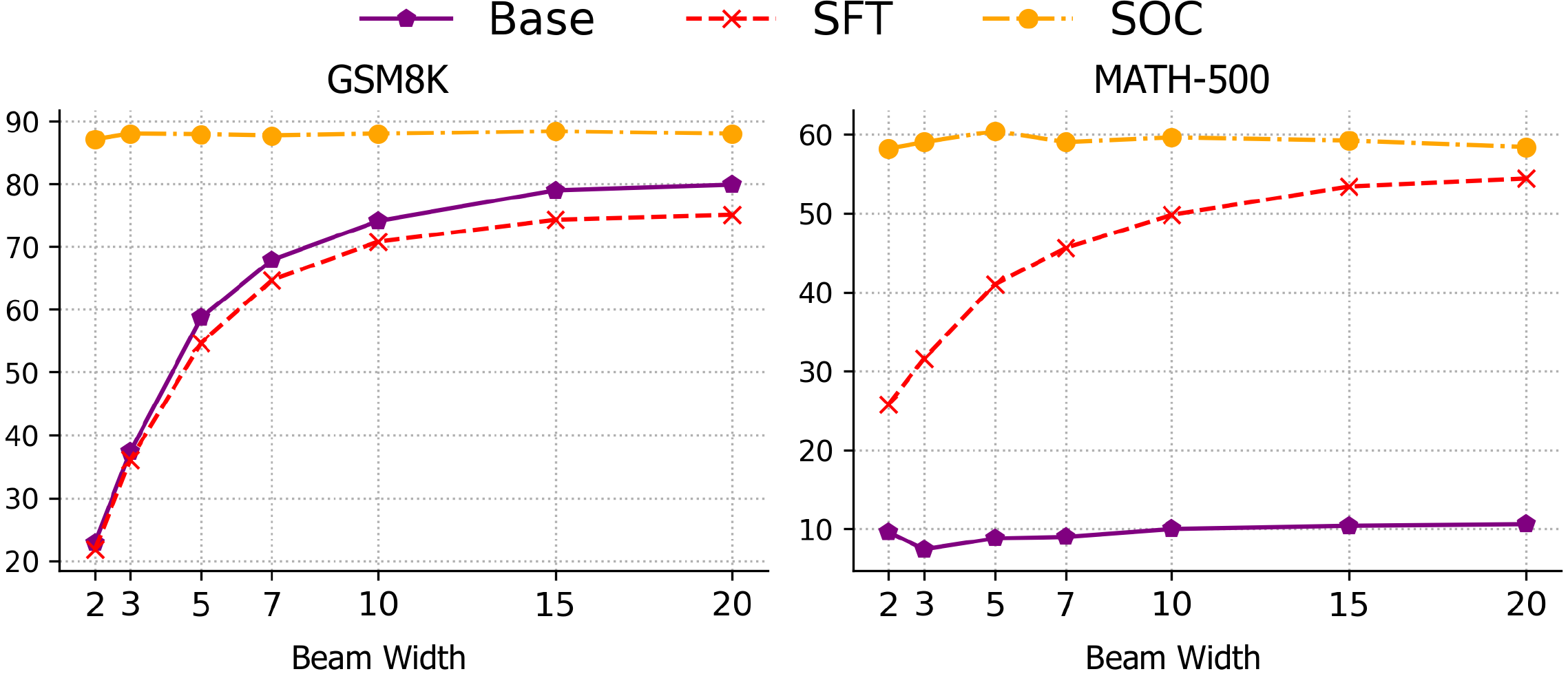}
  \caption {Performance of Qwen-2.5-7B-Base, SFT, and SOC across different beam widths. SOC consistently outperforms Base model and SFT across various beam widths on both benchmarks, showcasing better performance with less computation demand for the inference.}
    \label{fig: ablation}
\end{figure}

\paragraph{Impact of $\lambda$} In the main experiments, we set the hyperparameter $\lambda$ in Equation~\eqref{eq: v loss} to 0.2 in the results of Table~\ref{main-result}. To further analyze the impact of $\lambda$ on the performance of SOC, we conduct an ablation study by varying $\lambda$ and evaluating its effect on the performance of SOC. The results are shown in Table~\ref{ablation-lambda}.

We range $\lambda$ in \texttt{\{0.1, 0.2, 0.3\}} for 7B model with \texttt{DATA-$\beta$} and \texttt{DATA-$\alpha$}, and 3B model with \texttt{DATA-$\alpha$}, to analyze how $\lambda$ affects the performance of SOC for different datasets and models. In general, we find that the performance is not sensitive to the value of $\lambda$. Comparing 7B model with \texttt{DATA-$\beta$} and \texttt{DATA-$\alpha$}, we find that SOC is more effective and robust on \texttt{DATA-$\alpha$} than \texttt{DATA-$\beta$}, which can be due to the larger~(4times) dataset size of \texttt{DATA-$\alpha$}. On \texttt{DATA-$\alpha$}, 7B model gains more improvement than 3B model for all the $\lambda$ values and for both the best performance is achieved with $\lambda$=0.1.

We observe that SOC achieves the best performance when $\lambda$ is set to 0.2, which is consistent across different base models and datasets. This indicates that a moderate value of $\lambda$ effectively balances the influence of the auxiliary $V$ loss, leading to optimal performance improvements.

\begin{table*}[!t]
\caption{Results (accuracy \% of pass@1) on mathematical reasoning tasks. We report the results of each model sampling with beam search~(beam width of 5), to show the impact of different $\lambda$ on the performance of SOC.}
\label{ablation-lambda}
\vskip 0.15in
\begin{center}
\begin{sc}
\resizebox{\textwidth}{!}{%
\begin{tabular}{l|c|c|c|l}
\toprule
Model & GSM8K & MATH-500 & GaoKao-EN & Average \\
\midrule
\textit{Qwen-2.5-7B-Base} & 58.7 & 8.8 & 33.0 & 33.5 \\
\quad+SFT(\texttt{DATA-$\beta$}) & 79.9 & 41.0 & 44.9 & 55.3 \\
\quad+SFT(\texttt{DATA-$\beta$})+SOC($\lambda$=0.1) & 79.7 & 41.4 & 44.7 & 55.3\color{red}{(+0.0)} \\
\quad+SFT(\texttt{DATA-$\beta$})+SOC($\lambda$=0.2) & 83.7 & 41.4 & 45.7 & 56.9\color{red}{(+1.6)} \\
\quad+SFT(\texttt{DATA-$\beta$})+SOC($\lambda$=0.3) & 81.4 & 41.2 & 46.0 & 56.2\color{red}{(+0.9)} \\
\quad+SFT(\texttt{DATA-$\alpha$}) & 54.7 & 41.4 & 28.6 & 41.6 \\
\quad+SFT(\texttt{DATA-$\alpha$})+SOC($\lambda$=0.1) & 86.6 & 60.8 & 52.7 & 66.7\color{red}{(+25.1)} \\
\quad+SFT(\texttt{DATA-$\alpha$})+SOC($\lambda$=0.2) & 87.9 & 60.4 & 51.4 & 66.6\color{red}{(+25.0)} \\
\quad+SFT(\texttt{DATA-$\alpha$})+SOC($\lambda$=0.3) & 87.2 & 57.6 & 52.2 & 65.7\color{red}{(+24.1)} \\
\midrule
\textit{Qwen2.5-3B-Base} & 51.5 & 40.8 & 31.4 & 41.2 \\
\quad+SFT(\texttt{DATA-$\alpha$}) & 79.2 & 49.6 & 34.8 & 54.5 \\
\quad+SFT(\texttt{DATA-$\alpha$})+SOC($\lambda$=0.1) & 82.6 & 55.4 & 47.3 & 61.8\color{red}{(+7.3)} \\
\quad+SFT(\texttt{DATA-$\alpha$})+SOC($\lambda$=0.2) & 83.5 & 54.4 & 46.5 & 61.5\color{red}{(+7.0)} \\
\quad+SFT(\texttt{DATA-$\alpha$})+SOC($\lambda$=0.3) & 84.0 & 51.4 & 49.4 & 61.6\color{red}{(+7.1)} \\
\bottomrule
\end{tabular}
}
\end{sc}
\end{center}
\vskip -0.1in
\end{table*}

\section{Related Work}
\subsection{RL for LLM}
Previous studies have demonstrated that RL enhances the performance of LLMs by optimizing reward feedback. These approaches typically involve RL from Human Feedback~(RLHF)~\citep{stiennon2020learning,DPO,liu2025survey} to align LLMs with human preference, self-correction~\citep{kumar2024training}, or direct fine-tuning for reasoning ability from the base model~\citep{guo2025deepseek}. These techniques enable LLMs to generate more aligned, accurate, and coherent responses. While our work follows a similar analysis, particularly in the context of offline RL~\cite{kumar2020conservative,lyu2022mildly,qing2024a2po,zhang2023uncertainty,wen2024contrastive,wen2024towards}, we focus on the overestimation problem during SFT. We draw inspiration from offline RL to examine this problem within the SFT stage, offering a novel perspective on how overestimation impacts the optimization process. 

\subsection{Search-based methods of LLMs}
Search-based methods during test-time computation are crucial for enabling models to improve their output quality \citep{snell2024scaling,liu2025can}. These methods include beam search~\cite{snell2024scaling}, Best-of-N~(BoN) sampling~\cite{brown2024large}, and lookahead-search methods, like MCTS~\citep{zhang2406rest,xie2024monte}, generate sequences of $k$ steps and evaluate which paths to retain for further exploration. Among these methods, beam search is a simple and widely used search method without extra reward models or verifiers. 
\citet{li2022positive} first investigates the positive-incentive effects of noise and pioneers a series of studies on leveraging structured randomness to support task performance, offering valuable insights for test-time exploration strategies.
\citet{arora2022exposure} provides a theoretical and empirical analysis of exposure bias in text generation and \citet{wang2022self} demonstrates that self-consistency can improve the performance of LLMs, using random sampling (noise) to explore multiple reasoning trajectories and then aggregating them.
In this work, we examine how search procedures, particularly beam search, may suffer from over-optimism due to implicit maximization of $Q$-values during inference.
Our work focuses on refining beam search techniques by addressing the over-optimism problem, which can lead to inflated $Q$-value estimates and amplified reasoning errors, particularly in long-horizon or sparse-reward settings.

\section{Conclusion}

In this paper, we formulate LLMs as token-level MDPs and establish a theoretical equivalence between SFT and offline RL, where LLMs implicitly learn a $Q$-function. We further show that beam search, a widely used decoding method, relies on this implicit $Q$-function but suffers from over-optimism due to value estimation errors. Based on that, we propose SOC, a simple auxiliary loss applied during SFT. Despite its simplicity, SOC is theoretically proven to effectively correct optimism, leading to more reliable guidance in inference time. Extensive experiments on mathematical reasoning benchmarks demonstrate that SOC significantly enhances reasoning performance across state-of-the-art open-source models.
\section*{Limitations}
Although we conduct a comprehensive analysis of the over-optimism problem and the proposed method SOC, certain limitations remain, along with potential future research directions to explore:~(1)~While this work mainly focuses on the over-optimism problem of LLMs, it is valuable to investigate the issue of multi-modal models such as Visual Language Models~(VLMs).~(2)~Investigating whether other search-based methods of LLMs encounter this issue is another direction and important to the development of test-time computation.
\section*{Ethical Considerations}
We believe this work contributes to the development of LLMs in the field of NLP. It is worth mentioning that all the experiments are conducted using open-source models and datasets, ensuring no potential social concerns.

\section*{Acknowledgments}
This research is supported by the RIE2025 Industry Alignment Fund – Industry Collaboration Projects (IAF-ICP) (Award I2301E0026), administered by A*STAR, as well as supported by Alibaba Group and NTU Singapore through Alibaba-NTU Global e-Sustainability CorpLab (ANGEL).

\bibliography{custom}

\newpage
\appendix

\section{Proof of Theorem~\ref{thm:convergence_values}}
\label{app: proof}
\begin{theorem}[Contraction of Value Differences]
Let $V_\theta(s)$ be the approximate value function of state $s$. Suppose the value function is always positive for any state $s$. If the objective in SFT includes a additional term in Equation~\eqref{eq: v loss}, then after one step of gradient descent with learning rate $\alpha \in [0,1]$, the gap between adjacent states' values contracts, \textit{i.e.},  
\[
| V'_\theta(s_t) - V'_\theta(s_{t+1}) | \leq | V_\theta(s_t) - V_\theta(s_{t+1}) |, \quad \forall t,
\]
where $V'$ denotes the updated value function after one optimization step.
\end{theorem}
\begin{proof} Recall that the auxiliary objective is to minimize \( -\log V(s) \). After one step of gradient descent, the updated value function becomes \( V'(s) = V(s) + \alpha \cdot \frac{1}{V(s)} \). Now, we can compute the difference between the updated values for any adjacent states \( s_t \) and \( s_{t+1} \):
\begin{equation}
    \begin{aligned}
&\left| V'(s_t) - V'(s_{t+1}) \right|\\ & = \left| V(s_t) - V(s_{t+1}) + \alpha \left( \frac{1}{V(s_t)} - \frac{1}{V(s_{t+1})} \right) \right| \\
&\leq \left| 1 - \frac{\alpha}{V(s_t) V(s_{t+1})} \right| \cdot \left| V(s_t) - V(s_{t+1}) \right| \\
&\leq \left| V(s_t) - V(s_{t+1}) \right|.
    \end{aligned}
\end{equation}
The last inequality holds since we assume the value function is always positive and learning rate $\alpha \in [0,1]$. Thus, we see that after one step of gradient descent, the gap between the updated value functions of adjacent states is no greater than the original gap.
\end{proof}

\section{The Derivation for Estimation Error}
\label{app: The Derivation of Estimation Error}
We provide a proof sketch for proof of Equation~\eqref{eq: gradient of sft}. We start from SFT objective, which is given by:
\begin{equation}
  \mathcal{L}_{\text{SFT}} = \mathbb{E}_{s\sim \mathcal{D}}\left[-\sum_{a} \pi^*(a|s) \log \pi_\theta(a|s) \right].
\end{equation}
To compute its gradient, we first differentiate the softmax function:
\begin{equation}
  \pi_\theta(a|s) = \frac{\exp\big(Q_\theta(a|s)\big)}{\sum_{a'} \exp\big(Q_\theta(a'|s)\big)}.
\end{equation}
Taking the gradient with respect to the logits $Q_\theta(a|s)$, we obtain:
\begin{equation}
  \frac{\partial \pi_\theta(a|s)}{\partial Q_\theta(a'|s)} = \pi_\theta(a|s) \left( \mathbb{I}[a = a'] - \pi_\theta(a'|s) \right).
\end{equation}
Next, we differentiate the loss function:
\begin{equation}
  \nabla_\theta \mathcal{L}_{\text{SFT}}=\mathbb{E}_{s\sim \mathcal{D}}\left[ -\sum_{a} \pi^*(a|s) \frac{\partial \log \pi_\theta(a|s)}{\partial \theta} \right].
\end{equation}
Since
\begin{equation}
\frac{\partial \log \pi_\theta(a|s)}{\partial \theta} = \frac{1}{\pi_\theta(a|s)} \frac{\partial \pi_\theta(a|s)}{\partial \theta},
\end{equation}
we obtain:
\begin{equation}
-\sum_{a} \pi^*(a|s) \sum_{a'} \frac{1}{\pi_\theta(a|s)} \frac{\partial \pi_\theta(a|s)}{\partial Q_\theta(a'|s)} \frac{\partial Q_\theta(a'|s)}{\partial \theta} .
\end{equation}
Substituting the gradient of $\pi_\theta(a|s)$:
\begin{equation}
 -\sum_{a} \pi^*(a|s) \sum_{a'} \left( \mathbb{I}[a\!=\!a']\!-\!\pi_\theta(a'|s) \right) \frac{\partial Q_\theta(a'|s)}{\partial \theta}.
\end{equation}
Rearranging the terms, we obtain:
\begin{equation}
  \nabla_\theta  \mathcal{L}_{\text{SFT}}\!=\!\mathbb{E}_{s\sim \mathcal{D}}\left[ \sum_{a} \left( \pi_\theta(a|s)\!-\!\pi^*(a|s) \right) \frac{\partial Q_\theta(a|s)}{\partial \theta} \right].
\end{equation}
We finish the proof by replacing $\pi$ with normalized $\widetilde{Q}$, which is equivalent and unambiguous.

\section{Prompt for Evaluation}
\label{app: prompt}
Here we provide the prompt for evaluation in our experiments. Following previous work~\cite{qwen2.5math}, we use CoT with few-shots demonstrations to prompt LLMs to solve mathematical problems. The two kinds of templates are shown in Figure~\ref{fig: gsm_qcot} and Figure~\ref{fig: gsm_cot}.
\begin{figure}[h]
  \includegraphics[width=\linewidth]{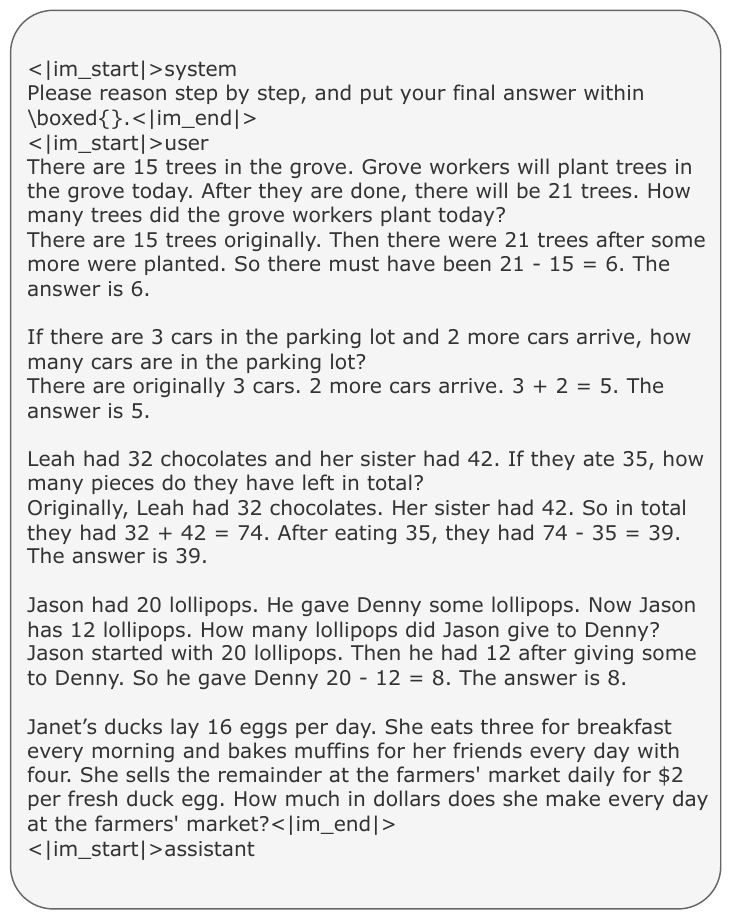}
  \caption {Template 1 for prompt.}
    \label{fig: gsm_qcot}
\end{figure}

\begin{figure}[h]
  \includegraphics[width=\linewidth]{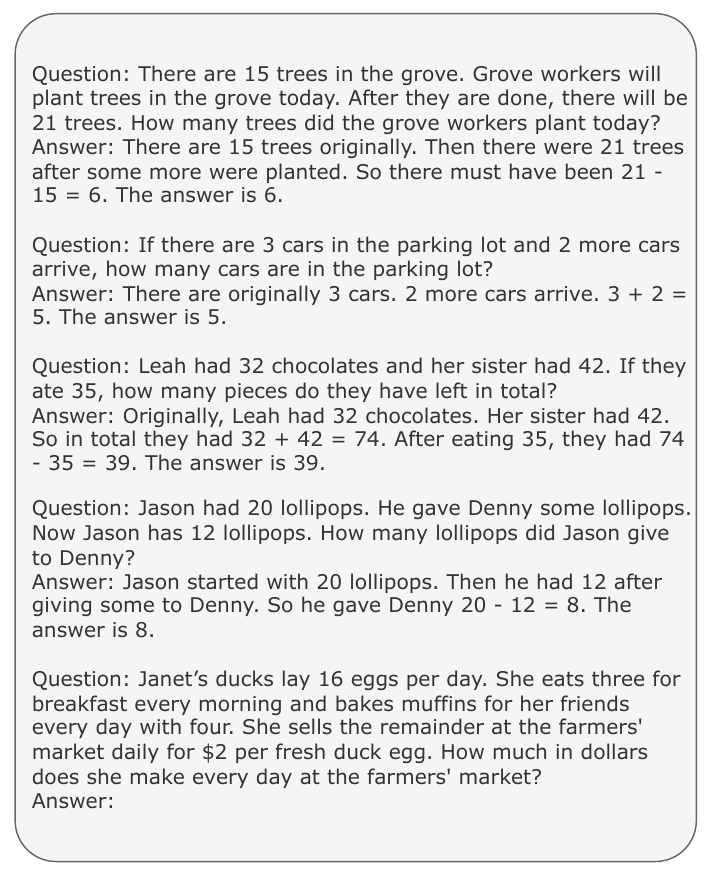}
  \caption {Template 2 for prompt.}
    \label{fig: gsm_cot}
\end{figure}

\end{document}